\documentclass[]{article}

\usepackage[utf8]{inputenc} % allow utf-8 input
\usepackage[T1]{fontenc}    % use 8-bit T1 fonts

\AtBeginDocument{%
  }

\usepackage{xcolor}
\usepackage{siunitx} % \num command
\usepackage{dsfont} % indicator function
\usepackage[most]{tcolorbox}
\usepackage{adjustbox}
\usepackage{bbold} % indicator function

\renewcommand{\l}[0]{k} % label flipping
\newcommand{\rr}{\epsilon} % robust radius
\newcommand{\thetaset}{\Theta}
\newcommand{\yp}{\tilde{\mathbf{y}}}
\newcommand{\intv}{\mathbb{\intdomain}}
\newcommand{\intdomain}{\mathbb{IR}}

\newcommand{\x}{\mathbf{X}}
\newcommand{\xtest}{\mathbf{x}}

\newcommand{\y}{\mathbf{y}}

\newcommand{\matrprod}{\mathbf{z}} % shorthand for (X^T X)^(-1) X^T x_new
\newcommand{\matrprodscal}{z}
\newcommand{\matrprodc}{\mathbf{C}} %shorthand for (X^T X)^(-1) X^T (approx. version of matr prod)

\newcommand{\Bias}[0]{\mathcal{M}} % bias model 

\usepackage{amsthm} % for proof environment

\newtheorem{example}{Example}[section]
\newtheorem{theorem}{Theorem}[section]

\usepackage{hyperref}
\usepackage[capitalize,noabbrev]{cleveref}

\usepackage{multicol}
\usepackage{multirow}
\usepackage{booktabs} % for \toprule, etc.

\usepackage{algorithm}
\usepackage[noend]{algpseudocode}
\algnewcommand\algorithmicforeach{\textbf{for each}}
\algdef{S}[FOR]{ForEach}[1]{\algorithmicforeach\ #1\ \algorithmicdo}

\newcommand{\yu}{\y^u}
\newcommand{\yl}{\y^l}
\newcommand{\yuscal}{y^u}
\newcommand{\ylscal}{y^l}

\usepackage{fancybox, framed} % https://mirror.mwt.me/ctan/macros/latex/contrib/framed/framed.pdf

\usepackage{authblk} % author formatting
\usepackage{microtype} %makes formatting nice apparently?
\usepackage{fullpage} %removes header with title

\DeclareMathOperator*{\argmax}{argmax}
\DeclareMathOperator*{\argmin}{argmin}
\renewcommand{\l}[0]{k} % label flipping

\title{The Dataset Multiplicity Problem: How Unreliable Data Impacts Predictions}

\author[]{Anna P.~Meyer}
\author[]{Aws Albarghouthi}
\author[]{Loris D'Antoni}
\affil[]{University of Wisconsin - Madison \protect\\ \texttt{\{annameyer,aws,loris\}@cs.wisc.edu}}

\begin{document}
\date{}
\maketitle

\begin{abstract}
We introduce dataset multiplicity, a way to study how inaccuracies, uncertainty, and social bias in training datasets impact test-time predictions. 
The dataset multiplicity framework asks a counterfactual question of what the set of resultant models  (and associated test-time predictions) would be if we could somehow access \textit{all} hypothetical, unbiased versions of the dataset. 
We discuss how to use this framework to encapsulate various sources of uncertainty in datasets' factualness, including systemic social bias, data collection practices, and noisy labels or features.
We show how to exactly analyze the impacts of dataset multiplicity for a specific model architecture and type of uncertainty: linear models with label errors. 
Our empirical analysis shows that real-world datasets, under reasonable assumptions, contain many test samples whose predictions are affected by dataset multiplicity. 
Furthermore, the choice of domain-specific dataset multiplicity definition determines what samples are affected, and whether different demographic groups are disparately impacted. 
Finally, we discuss implications of dataset multiplicity for machine learning practice and research, including considerations for when model outcomes should not be trusted. 
\end{abstract}

\section{Introduction}

Datasets that power machine learning algorithms are supposed to be accurate and fully representative of the world, but in practice, this level of precision and representativeness is impossible~\cite{jacobsMeasurement, raji2021ai}. 
Datasets display inaccuracies --- which we use as a catch-all term for both errors and nonrepresentativeness --- due to sampling bias~\cite{genderShades}, human errors in label or feature transcription~\cite{northcutt2021pervasive,vasudevanBagel}, and sometimes deliberate poisoning attacks~\cite{biggio2012poisoning,shafahi2018poison}. 
Datasets can also reflect undesirable societal inequities. 
But more broadly, datasets never reflect objective truths because the worldview of their creators is imbued in the data collection and postprocessing~\cite{jacobsMeasurement,paullada,raji2021ai}. 
Additionally, seemingly-trivial decisions in the data collection or annotation process influence exactly what data is included, or not~\cite{paullada, rechtImagenet}. In psychology, these minute decisions have been termed `researcher degrees of freedom,' i.e., choices that can inadvertently influence conclusions that one ultimately draws from the data analysis~\cite{simmonsPsychology}. 
In this paper, we study how unreliable data of all kinds impacts the predictions of the models trained on such data and frame this analysis as a `multiplicity problem.'

Multiplicity occurs when there are multiple explanations for the same phenomenon. 
Many recent works in machine learning have studied predictive multiplicity, which occurs when multiple models have equivalent accuracy, but still give different predictions to individual samples~\cite{costonFairness,predictive-multiplicity,semenovaExistence}.  
A consequence of predictive multiplicity is procedural unfairness concerns; namely, defending the choice of model may be challenging when there are alternatives that give more favorable predictions to some individuals~\cite{blackMultiplicity}.
But model selection is just one source of multiplicity.
In this paper, we argue that it makes sense to consider training datasets through a multiplicity lens, as well. To do so, we will consider a \emph{set of datasets}. 
Intuitively, this set captures all datasets that could have been collected if the world was slightly different, i.e., if we could correct the unknown inaccuracies in the data. 
We illustrate this idea through the following example.

\paragraph{Example dataset multiplicity use case}
Suppose a company wants to deploy a machine learning model to decide what to pay new employees. 
They have access to current employees' backgrounds, qualifications, and salaries.
However, they are aware  that in various societies, including the United States, there is a gender wage gap, i.e., systematic disparities in the average pay between men and women~\cite{censusWageGap}. 
Economists believe that while some of the gap is attributable to factors like choice of job industry, overt discrimination also plays a role~\cite{blauWageGap}. 
But even though we know that discrimination exists, it is very difficult to adjudicate whether specific compensation decisions are affected by discrimination. 

The original dataset is shown, along with the best-fit model $f$, in \Cref{fig:introex_v2}(a). 
Note that under $f$, the proposed salary for a new employee $\mathbf{x}$ is \$\num{73000}. 
But an alternate possibility of the `ground-truth,' debiased dataset is shown in \Cref{fig:introex_v2}(b). 
In the world that produced this dataset, perhaps there is no gender discrimination, so, ideally, we would learn from this dataset and yield model $g$, which places $\mathbf{x}$'s salary at \$\num{78000}. 

\begin{figure}
    \centering
    \includegraphics[width=14cm]{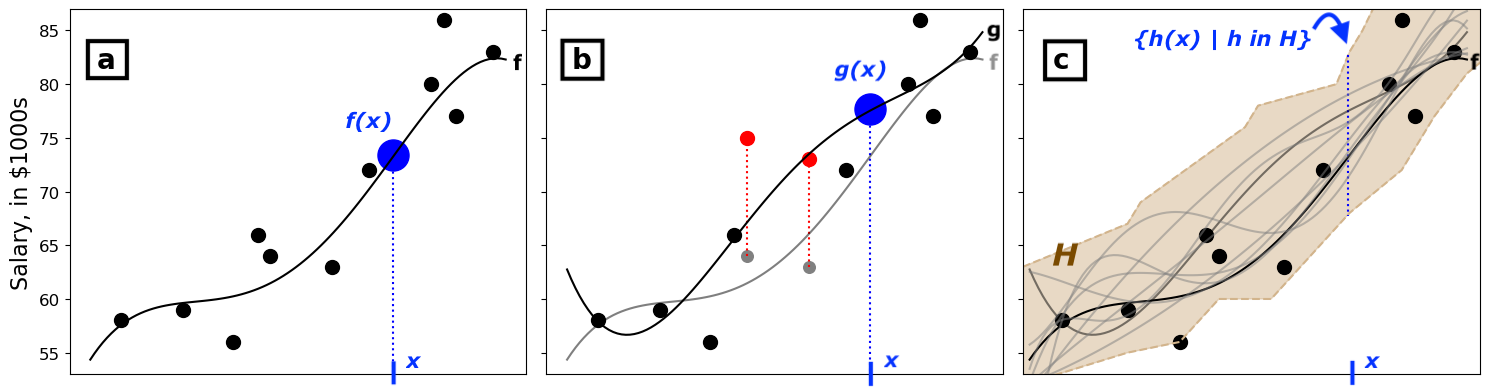}
    \caption{Salary prediction: (a) Training dataset and resultant model $f$. The prediction for the test sample $\mathbf{x}$ is \$\num{73000}. (b) Training dataset with two label modifications (in red) along with the newly-learned model ($g$). The prediction for the test sample $\mathbf{x}$ is now about \$\num{78000}. (c) $H$ contains the set of models $h$ that we could have obtained based on various small modifications to the provided dataset. We see that $\mathbf{x}$'s prediction can be anywhere between \$\num{68000} and \$\num{83000} (blue dotted line). }
    \label{fig:introex_v2}
\end{figure}

The modified dataset in \Cref{fig:introex_v2}(b) is just one example of how different data collection practices --- in this case, collecting data from an alternate universe where there is no gender-based discrimination in salaries --- can lead to various datasets that produce models that make conflicting predictions for individual test samples. 

But what if we could consider the entire range of candidate `ground-truth' datasets? 
For example, all datasets where each woman's salary may be increased by up to \$\num{10000} to account for the impacts of potential gender discrimination. 
\Cref{fig:introex_v2}(c) shows what we could hypothetically do if that set of datasets were available --- and we had unlimited computing power. 
Rather than outputting a single model, we bound the range of models (the highlighted region) obtainable from alternate-universe training datasets. We could then use this set of models to obtain a confidence interval for $\mathbf{x}$'s prediction - in this case, \$\num{68000}-\$\num{83000}. This range corresponds to 
the \emph{dataset multiplicity robustness} of $\mathbf{x}$, that is, the sensitivity of the model's prediction on $\mathbf{x}$ given specific types of changes in the training dataset. 

Work in algorithmic stability, robust statistics, and distributional robustness has attempted to quantify how varying the training data impacts downstream predictions. However, as illustrated by the example, we aim to find the \emph{pointwise} impact of uncertainty in training data rather than studying robustness purely in aggregate, and we want our analysis to encompass the \emph{worst-case} (i.e., adversarial) reasonable alternate models, unlike the purely statistical approaches. 

\paragraph{Our vision for dataset multiplicity in machine learning}
The proposed solution in the above example suffers from a number of drawbacks. 
First, is the solution solving the right problem? 
That is, is dataset multiplicity a better choice for reasoning about unreliable data than existing learning theory techniques? 
Second, how do we define what is a reasonable alternate-universe dataset to include in the set of datasets? 
Third, even if we had a set of datasets encompassing all alternate universes, how would we compute the graph in \Cref{fig:introex_v2}(c)? 
And finally, what are the implications for fair and trustworthy machine learning? How can, and should, we incorporate dataset multiplicity into machine learning research, development and deployment?

We address all of these concerns in this paper through the following contributions: 

\begin{description}
\item[Conceptual Contribution 1] We formally define the dataset multiplicity problem
and give several example use cases to provide intuition of how to define the set of reasonable alternate-universe
datasets (\Cref{sec:problemIntro})

\item[Theoretical Contribution] We present a novel technique that, for linear models with label errors, can exactly characterize the range of a test sample's prediction. We also show
how to over-approximate the range of models (i.e., \Cref{fig:introex_v2}(c)) (\Cref{sec:linear})

\item[Experimental Contribution] We use our approaches to evaluate the effects
of dataset multiplicity on real-world datasets with a particular eye towards how demographic subgroups are differently affected (\Cref{sec:results})

\item[Conceptual Contribution 2] We explore the implications of dataset multiplicity  (\Cref{sec:implications})
\end{description}

\section{The Dataset Multiplicity Problem}\label{sec:problemIntro}

We formalize \emph{dataset multiplicity} as a technique that conceptualizes uncertainty and societal bias in training datasets and discuss how to use dataset multiplicity as a tool to critically assess machine learning models' outputs.

% \paragraph{Problem setup}
We assume the following supervised machine learning setup: 
we start with a fixed, deterministic learning algorithm $A$ and a training dataset $D=(\x,\y)$ 
with features $\x\in\mathbb{R}^{n\times d}$ 
and outputs $\y\in\mathbb{R}^n$.\footnote{
$A$ is inclusive of all modeling steps including preprocessing the training data, selecting the model hyperparameters through a holdout validation set (segmented off from $D$), etc.
 }
We run $A$ on $D$ to get a model $f$, that is, $f=A(D)$. Given a test sample $\xtest$, we obtain an associated prediction  $\hat{y}=f(\xtest)$.

\subsection{Defining Dataset Multiplicity}\label{sec:definitions}
We describe dataset multiplicity for a dataset $D$ with a \emph{dataset multiplicity model} $\Bias(D)$. Intuitively, we want $\Bias(D)$ to be the smallest set that contains all conceivable alternatives to $D$. 
We present a few  examples of $\Bias$:

% \paragraph{$\Bias$ under data-poisoning} \anna{choose a less-extreme example?} Suppose a dataset $D$ was compromised and an attacker inserted up to $m$ poisoned samples. $\Bias(D)$ will be the set of all datasets obtainable by starting with $D$ and removing up to $m$ samples, i.e., 
% $$ \Bias(D) = \{D'\mid D'\subseteq D \land |D\setminus D'|\leq m \}$$

\paragraph{$\Bias$ under societally-biased labels} We continue the example from the introduction. Suppose we believe that  women in a dataset are underpaid by up to \$\num{10000} each. In this case, we define $\Bias(D)$ as the set of all datasets $D'$ with identical features to $D$, such that all labels for men in $D'$ are identical to the labels in $D$ and all labels for women in $D'$ are the same as in $D$, or increased by up to $\$\num{10000}$.  
% $$ \Bias(D) = \{(\x,\y')\mid (\x_i)_0=1\implies y'_i\in[y_i, y_i+\num{10000}]\textrm{ and } (\x_i)_0 \neq 1\implies y_i=y'_i \}$$

\paragraph{$\Bias$ under noisy measurement} Suppose a dataset contains a weight feature and where data was collected using a tool whose measurements may be inaccurate by up to 5 grams. If weight is the $k^{\textrm{th}}$ feature, then we can represent $\Bias(D)$ as the set of all datasets $D'$ that are identical to $D$ except that feature $k$ may differ by up to 5 grams.
% removing math-y terms because it was imprecise and probably not necessary to convey main ideas
% i.e., $|D|=|D'|$ and $(\mathbf{x},y)\in D \implies \exists (\mathbf{x}',y)\in D'$ with ${\mathbf{x}_k}'\in [{x_k}-5, {x_k}+5]$ and ${x_i}' = {x_i}$ for all indices $i\neq k$. 
% $$ \Bias(D) = \{D' \mid |D|=|D'| \land ((\mathbf{x},y)\in D\implies \exists (\mathbf{x}',y)\in D' \text{ s.t. } x_i=x'_i\forall i\neq k \land |x_k-x'_k|\leq 5)\} $$

\paragraph{$\Bias$ given unreliable feature data} People commonly misreport their height on dating profiles~\cite{datingToma}: men add 0.5"($\pm 0.88$) to their height, on average, while women add 0.17"($\pm 0.98$). So, for a given man with reported height $h$ we can be 95\% confident that his height is in $[h-1.26, h+2.26]$ and for a women with height $h$, $[h-1.79, h+2.13]$. If a dataset contains height (as self-reported via dating apps) in position $i$, then $\Bias(D)$ will contain all datasets $D'$
that are equivalent to $D$, except that each sample's $i^{\textrm{th}}$ feature may be modified according to the gender-specific 95\% confidence intervals.
% such that $|D|=|D'|$, and for all $(\mathbf{x},y)\in D$, there is some $(\mathbf{x'},y)\in D'$ such that ${x_j}' ={x_j}$ for all $j\neq i$ and $x'_i\in[x_i-1.26, x_i+2.26]$ if $\mathbf{x}$ is male and $x_i\in[x_i-1.79, x_i+2.13]$ if $\mathbf{x}$ is female.
% \begin{align*}
%     \Bias(D) = \{D' \mid & |D|=|D'|\land  (\mathbf{x},y)\in D \implies \exists (\mathbf{x}',y)\in D' \text{ s.t. } \\ &(x_r = x'_r \forall r\neq i \land \\  &(x_j=\text{male}\implies x'_i\in[x_i-1.26,x_i+2.26]) \land (x_j=\text{female}\implies x'_i\in[x_i-1.79,x_i+2.13]))\}
% \end{align*}

\paragraph{$\Bias$ with missing data}  Getting a representative population sample can be a challenge for many data collection tasks. 
Suppose we suspect that a dataset underrepresents a specific minority population by up to \num{20} samples. 
If the undersampled group has feature $i=k$, then $\Bias(D)$ will be the set of datasets $D'$  such that $D\subseteq D'$, there are at most 20 additional samples in $D'$, and all new samples have feature $i=k$. 
(We additionally assume that all new samples are in the proper feature space, i.e., that they are valid data samples.)
%$|D'\setminus D|\leq 20$, and ${x_i}=k$ for all $(\mathbf{x},y)\notin D$.  
% $$\Bias(D) = \{D' \mid D\subseteq D'\land |D'\setminus D|\leq 20 \land \forall \mathbf{x}\in (D'\setminus D). \mathbf{x}_i=k \}$$
%

\subsection{Learning with Dataset Multiplicity}
We define $A(\Bias(D))$ to be the set of all models obtainable by training with some dataset in $\Bias(D)$, i.e., $A(\Bias(D)) = \{f \mid \exists D'\in\Bias(D)\text{ s.t. }A(D')=f\}$. Given this set of models, we can inquire about the range of possible predictions for a test sample $\xtest$. In particular, we can ask whether $\xtest$ is \emph{robust} to dataset multiplicity, that is, will it receive a different prediction if we started with any other model in $\Bias(D)$? Formally, we say that a deterministic algorithm $A$, a dataset $D$, and a dataset multiplicity model $\Bias(D)$ are $\rr$-robust to dataset multiplicity on a sample $\mathbf{x}$ if \Cref{eq:robustness} holds. (Equivalently, we will say that $\mathbf{x}$ is $\rr$-robust.)
\begin{equation}\label{eq:robustness}
    D'\in\Bias(D)\implies A(D')(x) \in [A(D)(x)-\epsilon, A(D)(x)+\epsilon]
\end{equation}

\begin{example}
Returning to the example from the introduction, the test sample $\mathbf{x}$ originally receives a prediction of $\$\num{78000}$ (\Cref{fig:introex_v2}a). \Cref{fig:introex_v2}c shows that $\mathbf{x}$ is \emph{not} $\rr$-robust for $\rr=\$\num{5000}$, since it can receive any prediction in $[\num{68000},\ \num{83000}]$, and $\num{78000}-\num{68000}>\num{5000}$. However, $\mathbf{x}$ is $\rr$-robust for $\rr=\num{10000}$. 
\end{example}

If a sample $\mathbf{x}$ is $\rr$-robust, then we can be certain that its prediction will not change by more than $\rr$ due to dataset multiplicity. In practice, this may mean we can deploy the prediction with greater confidence, or less oversight. Conversely, if $\mathbf{x}$ is not $\rr$-robust, then this means there is some plausible alternate training dataset that, when used to train a model, would result in a different prediction for $\mathbf{x}$. In this case, the prediction on $\mathbf{x}$ may be less trustworthy --- we discuss options for dealing with non-robustness in  \Cref{sec:implications}.

\subsection{Choosing a Dataset Multiplicity Model} 
We have discussed how to formalize $\Bias(D)$ given various conceptions of dataset inaccuracy; however, we have not discussed how to determine in what ways a given dataset may be inaccurate. In practice, these judgments should be made in collaboration with domain experts, both within the data science and social science realms. Still, there is no one normative, `right' answer of how to define $\Bias$ for a given situation --- any judgment will be normative. Furthermore, there may be multiple ways to describe the same social phenomenon, as illustrated by the following example:

\begin{example}
    The first example in \Cref{sec:definitions} formalizes gender discrimination in salaries. When index 0 is gender and value 1 is woman, we define $\Bias$ as $\Bias(D) = \{(\x,\y')\mid (\x_i)_0=1\implies y'_i\in[y_i, y_i+\num{10000}]\textrm{ and } (\x_i)_0 \neq 1\implies y_i=y'_i\}$. However, what if we frame the problem as men are overpaid, rather than women are underpaid? In that case, a more appropriate formalization would be $\Bias(D) = \{(\x,\y')\mid (\x_i)_0=0\implies y'_i\in[y_i-\num{10000}, y_i]\textrm{ and } (\x_i)_0 \neq 0\implies y_i=y'_i\}$.
\end{example}

As we will see in \Cref{sec:results_demo}, this variability in framing can affect the conclusions we draw about dataset multiplicity's impacts, highlighting the need for thoughtful reflection and interdisciplinary collaboration when choosing $\Bias$.

% Even though there is no magic formula, we believe that forming $\Bias$ is a valuable reflection exercise for those involved in the modeling process. Furthermore, $\Bias$ can be adapted depending on the goals of the analysis, as we discuss further in \Cref{sec:implications}. 

\section{The Dataset Multiplicity Problem for Linear Models with Label Errors}
\label{sec:linear}

We consider a special case of the dataset multiplicity problem introduced in \Cref{sec:problemIntro}, namely, linear models given noise or errors in the training data's labels. (We use the term `label' in the context of both linear regression and classification.) We present this analysis to begin to characterize the impact of dataset multiplicity on real-world datasets and models, and to provide an example for how we envision the study of dataset multiplicity's impacts to continue in future work.

We focus on linear models with label errors for a few reasons. First, linear models are well-studied and used in practice, especially with tabular data, which is common in areas with societal implications. Furthermore, complicated models like neural nets can often be conceived of as encoders followed by a final linear layer, making our results more widely applicable. Second, label errors and noise are common, well-documented realities in many applications~\cite{northcutt2021pervasive}. Finally, as we will see, the closed-form solution for linear regression allows us to solve this problem exactly, a challenge that is currently impractical even for other simple, widely-studied model families like decision trees~\cite{meyer}.

\subsection{Formulating the Dataset Label Multiplicity Problem}
We assume our dataset is of the form $D=(\x,\y)$ with feature matrix $\x\in\mathbb{R}^{n\times d}$ and output vector $\y\in\mathbb{R}^n$. (Even though $\y$ is continuous, we borrow terminology from classification to also refer to $y_i$ as the \emph{label} for $\mathbf{X}_i$.) At times, we will reference the interval domain, $\intv$, that is, $\intv=\{[a,b]\mid a,b\in\mathbb{R}, a\leq b\}$. 

\paragraph{Parameterizing $\Bias$ given label perturbations} We parameterize $\Bias$ given label noise with three parameters, $k$, $\Delta$, and $\phi$.  First, $k\in\mathbb{N}$ is an upper bound on the number of training samples that have an inaccurate label. Second, $\Delta\in\intv^n$ stores 
the amount that each label can change. The $i^\textrm{th}$ element of $\Delta$ is $[\delta_i^l,\delta_i^u]$, signifying that the true value of $y_i$ falls in the interval $[y_i+\delta_i^l, y_i+\delta_i^u]$. 
Finally, $\phi$ is predicate over the feature space specifying whether we can modify a given sample when, for example, label errors are limited to a population subgroup. 
Given $k$, $\Delta$, and $\phi$ we define $\Bias$ as 
$$ \Bias_{k,\Delta,\phi}((\x,\y)) = \{ (\x,\y') \mid \|\mathbb{1}[ \y \neq \y']\|_1 \leq k \land \forall i.y_i\neq y_i' \implies \phi(\mathbf{x}_i)  \land \forall i. y_i' - y_i \in \delta_i \}$$

We describe $k$, $\Delta$, and $\phi$ for the following examples: 
\begin{example} We assume that women in a salary dataset are underpaid by up to $\$\num{10000}$ each. 
 Since we place no limit on how many labels may be incorrect --- beyond the proportion of women in the dataset --- we set $k=n$, the total number of samples. Since labels may be underreported by up to $\$\num{10000}$, we set $\Delta=[0,\num{10000}]^n$.  And finally, since $\mathbf{x}_0=1$ means that $\mathbf{x}$ is a woman, we define $\phi(\mathbf{x}) = \mathbb{1}[ \mathbf{x}_0 = 1]  $ since we assume that only women's salaries may change.
\end{example}

\begin{example} (Spam filter) Suppose a dataset $D = (\x, \y)$ contains emails $\x$ that are labeled as not spam or spam (i.e., $\y\in \{-1,1\}^n$). From manual inspection of a small portion of the dataset, we estimate that up to 2\% of the emails are mislabeled. Since up to $2\%$ of the labels may be incorrect, we set $k=0.02n$. As the labels are binary, we can modify each label by $+/- 2$, depending on its original label.\footnote{The linearity of the algorithm ensures that all labels will remain valid, i.e., either -1 or 1.} Thus, $\Delta = [a,b]^n$ where $[a,b]_i=[0,2]$ when $y_i=-1$ and $[a,b]_i=[-2,0]$ when $y_i=1$. Finally, $\phi(\mathbf{x})=1$ since there are no limitations on which samples have inaccurate labels.
\end{example}

\subsection{Linear Regression Overview}
Our goal is to find the optimal linear regression parameter $\theta$, i.e., 
\begin{equation}\label{eq:linear}
    \theta=\argmin_{\theta\in\mathbb{R}^d} (\theta^T \x - \y)^2 
\end{equation}
Least-squares regression admits a closed-form solution\footnote{In practice, we implement ridge regression, $\theta=(\x^T\x-\lambda I)^{-1}\x^T\y$, for greater stability.}
\begin{equation}\label{eq:leastsquares}
    \theta = (\x^\top \x)^{-1} \x^\top \y
\end{equation}
for which we will analyze dataset label multiplicity.
We work with the closed-form solution, instead of a gradient-based one, as it is deterministic and holistically considers the whole dataset, allowing us to exactly measure dataset multiplicity by exploiting linearity (Rosenfeld et al. make an analogous observation~\cite{rosenfeld-randomized}). On medium-sized datasets and modern machines, computing this closed-form solution is efficient. 

Given a solution, $\theta$, to \Cref{eq:linear}, we output the prediction $\hat{y} = \theta^\top\xtest$ for a test point $\xtest$. 

\paragraph{Extension to binary classification}
Given a binary output vector $\y\in\{-1,1\}^n$, we find $\theta$ in the same way, but when making test-time predictions, use 0 as a cutoff between the two classes, i.e., given parameter vector $\theta$ and test sample $\xtest$, we return $1$ if $\theta^T\xtest>0$ and $-1$ otherwise. To evaluate robustness for binary classification, we care about whether the predicted label can change when training with any dataset $D'\in\Bias(D)$. Thus, if $\theta^T\xtest \geq 0$, then $\xtest$ is dataset multiplicity robust if there is no model $\theta'$ such that $(\theta')^T\xtest < 0$ (and vice-versa when $\theta^T\xtest <0$).

\subsection{Exact Pointwise Solution}\label{sec:exact}
Given a model $\theta$ and test sample $\xtest$ we can expand and rearrange $\theta^\top \xtest$ as follows:
$$
   \hat{y} = \theta^\top \mathbf{x} = ((\x^\top \x)^{-1} \x^\top \y )^\top \xtest 
                     % &= ((\x^T \x)^{-1} \x^T)^T \xnew^T)^T \y 
                     = \underbrace{(\xtest^\top (\x^\top \x)^{-1}\x^\top)}_{\matrprod}\y
$$
This form is useful since it isolates $\y$, which under our dataset multiplicity assumption contains all of the dataset's uncertainty. We will use $\matrprod$ to denote $\mathbf{x}^\top (\x^\top \x)^{-1}\x^\top$.
Thus, our goal is to find
\begin{equation}\label{eq:yhat}
   (\_\ , \y') = \argmax_{(\_\ , \yp) \in\Bias_{\l,\Delta,\phi}(D)} |\matrprod \yp - \matrprod\y| 
\end{equation}
 and then to check whether $\|\hat{y} - \matrprod\y'\|< \rr$. If so, then we will have proved that the prediction for $\xtest$ is $\rr$-robust under $\Bias_{\l,\Delta,\phi}(D)$. (Conversely, if $\|\hat{y} - \matrprod\y'\|\geq \rr$, then $\y'$ is a counterexample proving that $\xtest$ is not $\rr$-robust under $\Bias_{\l,\Delta,\phi}(D)$.)

\paragraph{Solving for \Cref{eq:yhat}} One option is to formulate \cref{eq:yhat} as a mixed-integer linear program. However, due to the vast number of possible label perturbation combinations, this approach is prohibitively slow (we provide a runtime comparison with our approach in \cref{app:more-experiments}). 
Instead, we use the algorithmic technique presented in 
\Cref{alg:exact}. Intuitively, one iteration of the algorithm's inner loop identifies what output $y_i\in\y$ to modify so that we maximally increase $\matrprod\y$. 
After $k$ output modifications --- or once all outputs eligible for modification under $\phi$ have been modified --- we check whether the new prediction, $\hat{y}'$, has $\hat{y}'> \theta^T\xtest + \rr$. If this is the case, we stop because we have shown that $\xtest$ is not $\rr$-dataset multiplicity robust. Otherwise, we repeat a variation of the algorithm (see the appendix) to maximally decrease $\matrprod\y$ and check whether we can achieve $\hat{y}'<\theta^T \xtest - \rr$. 

\begin{algorithm}[!t]
    \caption{Find the allowable perturbation that makes $\hat{y}$  as large as possible, i.e., $\max_{(\x,\yp)\in\Bias_{\l,\Delta,\phi}((\x,\y))} \matrprod\yp $}
    \begin{algorithmic}
        \Require $\matrprod\in\mathbb{R}^n, (\x,\y)\in(\mathbb{R}^{n\times d},\mathbb{R}^n), \Delta\in\intdomain^n$ with $0\in\Delta$, $\l\geq 0$, $\phi:\mathbb{R}^d\rightarrow\{0,1\}$
        \State $\yu\gets\y$
        \If{$\matrprodscal_i\geq 0$}
             $\rho^+_i \gets \matrprodscal_i\delta_i^u\ $
        \textbf{else}
             ~$\rho^+_i \gets \matrprodscal_i\delta_i^l\ $
        \EndIf
        \If{not $\phi(\mathbf{x}_i)$}
            $\rho^+_i \gets 0 $
        \EndIf
        \State Let $\rho^+_{i_1},\ldots,\rho^+_{i_l}$ be the $l$ largest elements of $\rho^+$ by absolute value
        \ForEach{$\rho^+_{i_j}$}
            \If{$\matrprodscal_{i_j}\geq 0$}
                 ${\yuscal_{i_j}} \gets {\yuscal_{i_j}} + \delta_{i_j}^u$
            \textbf{else}~
                 ${\yuscal_{i_j}}\gets {\yuscal_{i_j}} + \delta_{i_j}^l$
            \EndIf
        \EndFor
        \State
        \Return $\matrprod\yu$
    \end{algorithmic}\label{alg:exact}
\end{algorithm}

\paragraph{Extension to binary classification}
The binary classification version of the algorithm works identically, except we check whether $\hat{y}$  rounds to a different class than $\theta^T\xtest$ to ascertain $\xtest$'s robustness.

\subsection{Over-Approximate Global Solution}\label{sec:approx}

In \Cref{sec:exact}, we described a procedure to find the exact dataset multiplicity range for a test point $\mathbf{x}$.
For every input $\mathbf{x}$ for which we want to know the dataset multiplicity, the procedure effectively relearns the worst-case linear regression model for $\mathbf{x}$.
In practice, we may want to explore the dataset multiplicity of a large number of samples, e.g., a whole test dataset, or we may need to perform online analysis.

We would like to understand the dataset multiplicity range for a large number of test points without performing linear regression for every input. % The following example provides intuition for our approach.
%
% \begin{example} \label{ex:approx_running}
% We continue the scenario from \Cref{sec:introex}. \Cref{fig:introex_v2}(c) shows the entire range of (non-linear) models that are a best-fit for allowable perturbations of the toy dataset. Given the test sample $\mathbf{x}$ shown in \Cref{fig:introex_v2}(a\&b), we can use (c) to see that the range of possible predictions for $\mathbf{x}$ is approximately $\$\num{68000}-\$\num{83000}$.  
% %Suppose we do not know the test point $\mathbf{x}$ in advance, but we want to understand the range of models that are possible so that we can quickly check robustness for new test points. 
% \end{example} 
%
We formalize capturing all linear regression models we may obtain as follows:
$$\thetaset = \{\theta\mid \theta = (\x^\top \x)^{-1} \x^\top \yp \textrm{ for some } (\x,\yp)\in\Bias_{\l,\Delta,\phi}(D)\}$$

To see whether $\xtest$ is  
$\rr$-robust, we check whether $\tilde{\theta}^\top \mathbf{x} \in [\theta^\top \mathbf{x}-\rr, \theta^\top \mathbf{x}+\rr]$ for all $\tilde{\theta}\in\thetaset$.
For ease of notation, let $\matrprodc=(\x^\top\x)^{-1}\x^\top$.
Note that $\matrprodc\in\mathbb{R}^{m\times n}$, while $\matrprod\in\mathbb{R}^{1\times n}$.

\paragraph{Challenges}
The set of weights $\thetaset$ is not enumerable and is non-convex (proof in appendix \ref{app:algs}). 
Our goal is to represent $\thetaset$ efficiently so that we can simultaneously apply
all weights $\theta \in \thetaset$ to a point $\mathbf{x}$.
Our key observation is that we can easily compute a hyperrectangular over-approximation of $\thetaset$.
In other words, we want to compute a set $\thetaset^a$ such that $\thetaset \subseteq \thetaset^a$.
Note that the set $\thetaset^a$ is an interval vector in $\intv^n$, since interval vectors represent hyperrectangles in Euclidean space.

This approach results in an overapproximation of the true dataset multiplicity range for a test sample $\xtest$ --- that is, some values in the range may not be attainable via any allowable training label modification.   

\paragraph{Approximation approach}
We will iteratively compute components of the vector $\thetaset^a$ 
by finding each coordinate $i$ as the following interval, where $\mathbf{c}_i$ are the column vectors of $\matrprodc$:
$$ 
\thetaset^a_i = \left[\min_{(\x,\y')\in\Bias_{\l,\Delta,\phi}(D)} \mathbf{c}_i \y', \max_{(\x,\y')\in\Bias_{\l,\Delta,\phi}(D)} \mathbf{c}_i \y'\right]
$$
To find $\min_{\y'\in\Bias_{\l,\Delta}(\y)} \mathbf{c}_i\y'$, we use the same process as in \Cref{sec:exact}.
Specifically, we use \Cref{alg:exact} to compute the lower and upper bounds of each $\thetaset^a_i$.
We show in the appendix that the interval matrix $\thetaset^a$ is the tightest possible hyperrectangular overapproximation of the set $\theta$.

\paragraph{Evaluating the impact of dataset multiplicity on predictions}
Given $\thetaset^a$ as described above, the output for an input $\mathbf{x}$ is provably robust to dataset multiplicity if
\begin{align}\label{eq:approx}
    (\thetaset^a)^\top \mathbf{x} \subseteq [\theta^\top \mathbf{x} - \rr, \theta^\top \mathbf{x} + \rr]
\end{align} 
Note that $(\thetaset^a)^\top \mathbf{x}$ is computed using standard interval arithmetic,
e.g., $[a,b] + [a',b'] = [a+a' , b+b']$.
Also note that the above is a one-sided check: we can only say that the model's output given $\mathbf{x}$ is robust to dataset multiplicity, but because $\thetaset^a$ is an overapproximation, if \Cref{eq:approx} does not hold,  we cannot conclusively say that the model's prediction on $\mathbf{x}$ is subject to dataset multiplicity. 
\section{Empirical Evaluation}\label{sec:results}

We use Python to implement the algorithms from \Cref{sec:exact,sec:approx} for measuring label-error multiplicity in linear models.\footnote{Our code is available at \url{https://github.com/annapmeyer/linear-bias-certification}.} To speed up the evaluation, we use a high-throughput computing cluster. (We request 8GB memory and 8GB disk, but all experiments are feasible to run on a standard laptop.)
This approach does not have a direct baseline with which to compare, as ours is the first work to propose and analyze the dataset-multiplicity problem.

\paragraph{Datasets and tasks}
We analyze our approach on three datasets: the Income prediction task from the FolkTables project \cite{folktables}, the Loan Application Register (LAR) from the Home Mortgage Disclosure Act publication materials \cite{hdma}, and MNIST 1/7 (i.e., the MNIST dataset limited to 1's and 7's) \cite{mnist}.  We divide each dataset into train (80\%), test (10\%), and validation (10\%) datasets and repeat all experiments across 10 folds, except when a standard train/test split is provided, as with MNIST.  %
We perform classification on the Income dataset (whether or not an individual earned over \$\num{50000}), on LAR (whether or not a home mortgage loan was approved), and on MNIST (binary classification limited to 1's and 7's). Additionally, in the appendix we evaluate the regression version of Income by predicting an individual's exact income. 
For all of the Income experiments, we limit the dataset to only include data from a single U.S. state to speed computations. 
In \Cref{sec:results_exact,sec:res_approx} we present results from a single state, Wisconsin, while in \Cref{sec:results_demo} we compare results across five different US states.

\paragraph{Accuracy-Robustness Tradeoff} 
There is a tradeoff between accuracy and robustness to dataset multiplicity that is controlled by the regularization parameter $\lambda$ in the ridge regression formula $\theta = (\x^\top \x- \lambda I)^{-1}\x^\top\y$.  Larger values of $\lambda$ improve robustness at the expense of accuracy. \Cref{fig:tolerance} illustrates this tradeoff. All results below, unless otherwise stated, use a value of $\lambda$ that maximizes accuracy. %  

\input{graphs/tols}

\paragraph{Experiment goals}
Our core objective is to see how robust linear models are to dataset label multiplicity. We measure this sensitivity with the \emph{robustness rate}, that is, the fraction of test points that receive invariant predictions (within a radius of $\rr$) given a certain level of label inaccuracies. The robustness rate is a proxy for the stability of a modeling process under dataset multiplicity, so knowing this rate --- and comparing it across various datasets, demographic groups, and algorithms --- can help ML practitioners analyze the trustworthiness of their models' outputs. In \Cref{sec:results_exact}, we describe the overall robustness results for each dataset. Then, in \Cref{sec:results_demo}, we perform a stratified analysis across demographic groups and show how varying the dataset multiplicity model definition can significantly change data's vulnerability to dataset multiplicity. Finally, in \Cref{sec:res_approx}, we discuss results of the over-approximate approach and how it can be used to evaluate dataset multiplicity robustness.

\subsection{Robustness to Dataset Multiplicity}\label{sec:results_exact}

\begin{tcolorbox}[boxsep=1pt,left=3pt,right=3pt,top=2pt,bottom=2pt]
\textbf{Key insights:}  When a small percentage (e.g., 1\%) of labels are incorrect, a significant minority of test samples are not robust to dataset multiplicity, raising questions about the reliability of the models' predictions.
\end{tcolorbox}

\Cref{tab:all_data} shows the fraction of test points that are dataset multiplicity robust for classification datasets at various levels of label inaccuracies. For each dataset, the robustness rates are relatively high ($>80\%$) when fewer than 0.25\% of the labels can be modified, and stay above 50\% for 1\% label error.

\begin{table}[t]
\small
\centering
\caption{Robustness rates (percentage of test dataset whose predicted label cannot change under dataset label multiplicity) for classification datasets given different rates of inaccurate labels.}
\label{tab:all_data}
\begin{tabular}{l|rrrrrrrrrrr}
\toprule
\multirow{2}{*}{Dataset} & \multicolumn{11}{c}{Inaccurate labels as a percentage of training dataset size}  \\
 & 0.1\%  & 0.25\%  & 0.5\%  & 0.75\%  & 1.0\%  & 1.5\% & 2.0\%  & 3.0\%  & 4.0\% & 5.0\% & 6.0\% \\\midrule
LAR  & 93.9 & 88.1 & 79.4 & 69.2 & 61.3 & 45.9 & 33.7 & 16.2 & 5.2 & 0.4 & 0.0 \\
Income & 91.1 & 81.4 & 67.8 & 58.4 & 50.7 & 37.2 & 23.3 & 12.1  & 4.8 & 1.7 & 0.7  \\
MNIST 1/7 & 98.3 & 96.3 & 93.1 & 88.3 & 84.4 & 73.1 & 60.8 & 38.8 & 23.3 & 13.1 & 7.0 \\
\bottomrule   
\end{tabular}
\end{table}

Despite globally high robustness rates, we must also consider the non-robust data points. In particular, we want to emphasize that for Income and LAR, each non-robust point represents an individual whose classification hinges on the labels of only a small number of training samples. That is, given the assumed uncertainty about the labels' accuracy, it is plausible that a `clean' dataset would output different test-time predictions for these samples. Some data points will almost surely fall into this category --- if not, that would mean the model was independent from the training data, which is not our goal! However, if a sample is not robust to a small number of label modifications, perhaps the model should not be deployed on that sample. Instead, if the domain is high-impact, the sample could be evaluated by a human or auxiliary model (see \Cref{sec:implications} for more discussions on how to handle non-robust test samples). 

Returning to \Cref{tab:all_data}, many data points are not robust at low label error rates, e.g., when 1\% of labels may be wrong, 49.3\% of Income test samples are not robust. Likewise, $38.7\%$ of LAR test samples can receive the opposite classification if the correct subset of 1\% of labels change. These low robustness rates call into question the advisability of using linear classifiers on these datasets unless one is confident that label accuracy is very high.

% 

% Table 2: similar, but specifically for regression/income
% now in appendix

% Graphs: Show how robustness is different for different demographic groups

\subsection{Disparate Impacts of Dataset Multiplicity}\label{sec:results_demo}

In \Cref{sec:results_exact}, we showed dataset multiplicity robustness results given the assumption that all labels in the training dataset were potentially inaccurate. However, in practice, label errors may be systemic. In particular, two of the datasets we analyzed in \Cref{sec:results_exact} contain data that may display racial or gender bias. 
We hypothesize that the Income dataset likely reflects trends where women and people of color are underpaid relative to white men in the United States, and that the LAR dataset may similarly reflect racial and gender biases on the part of mortgage lending decision makers. To leverage this refined understanding of potential inaccuracies in the labels, in this section we evaluate test data robustness under the following two \emph{targeted} dataset multiplicity paradigms: \begin{itemize}
    \item \emph{`Promoting' the disadvantaged group}: We restrict label modification to members of the disadvantaged group (i.e., Black people or women); furthermore, we only change labels from the negative class to the positive class.
    \item \emph{`Demoting' the advantaged group}: We restrict label modification to the advantaged group (i.e., White people or men); furthermore, we only allow change labels from the positive class to the negative class. This setup is compatible with the worldview (for example) that men are overpaid. 
    % \item Income where women (resp., Black people) are underpaid: we restrict label flipping to modify women (resp., Black people) with the negative class ($<\$\num{50000}$) to have the positive class ($\geq \$\num{50000}$).
    % \item Income where men (resp., White people) are overpaid: we allow label modifications only to change positively-labeled men (resp., White people) to the negative label.
    % \item HMDA where women (resp. Black people) are more likely to have their mortgage loan unfairly denied. Thus, we allow changing labels from -1 to 1 for samples belonging to this group.
    % \item HMDA where men (resp. White people) are more likely to have their mortgage erroneously approved. Thus, we allow changing labels from the positive to negative class just for this group.
\end{itemize}

Before delving into the results, we want to acknowledge that this analysis has a few shortcomings. First, for simplicity we use binary gender (male/female) and racial (White/Black) breakdowns. Clearly, this dichotomy fails to capture complexities in both gender and racial identification and perceptions. Second, the targeted dataset multiplicity models that we use also over-simplify both how discrimination manifests and how it can interact with other identities not captured by the data. Finally, we are not social scientists or domain experts and it is possible that the folk wisdom we rely on to propose data biases does not fully capture the patterns in the world. Rather, readers should treat this section as an analysis of `toy phenomena' meant to illustrate how our technique can be used for real-world tasks.

\paragraph{Basic results}
First, we present the overall dataset multiplicity robustness rates for the various multiplicity definitions. 
\begin{tcolorbox}[boxsep=1pt,left=3pt,right=3pt,top=2pt,bottom=2pt]\textbf{Key insights:} Limiting label errors to a subset of the training dataset (i.e., refining $\Bias$) yields higher dataset multiplicity robustness rates. However, the exact choice of $\Bias$ is significant.
\end{tcolorbox}
\Cref{fig:demo_rob} shows that in all cases the targeted multiplicity definition yields significantly higher overall robustness rates than a broad multiplicity definition does. Notably, limiting all label perturbations to one racial group for Income greatly affects robustness: using the original multiplicity definition (no targeting), no test samples are robust when 12\% of labels can be modified. However, when limiting label errors to Black people with the negative label, the robustness rate remains over 95\% --- it turns out this is not surprising, since fewer than 2\% of the data points have race=Black. However, more than 80\% of samples have race=White, and limiting label changes to White people with the positive label still yields over 60\% robustness when 12\% of labels can be changed. 
Similarly, using targeted dataset multiplicity definitions for LAR can increase overall robustness by up to 30 percentage points.

\input{graphs/demo_robustness.tex}

\paragraph{Demographic group robustness rates}
We also investigate the robustness rates for different demographic groups, both under the original, untargeted multiplicity assumptions and under the targeted versions. 
\begin{tcolorbox}[boxsep=1pt,left=3pt,right=3pt,top=2pt,bottom=2pt]
\textbf{Key insights:}  Different demographic groups do not exhibit the same dataset multiplicity robustness rates and targeting $\Bias$ to reflect real-world uncertainty can exacerbate discrepancies.
\end{tcolorbox}
Each row of \Cref{fig:statestargeting} compares baseline (untargeted) robustness rates, stratified by demographic groups, with targeted versions of $\Bias$ for five US states. We observe two trends: first, there are commonly racial and gender discrepancies (see the pairs of dotted lines). E.g., for all states except Wisconsin, men consistently have higher baseline robustness rates than women (sometimes by a margin of over 20\%). 
Second, using various targeted versions of $\Bias$ has unequal impacts across demographic groups.
The top row of \Cref{fig:statestargeting} shows that targeting on race=Black (i.e., allowable label perturbations can change Black people's labels from $-1$ to $1$) modestly improves dataset multiplicity robustness rates for Black people, but massively improves them for White people. We see similar trends, namely, that the non-targeted group sees higher robustness rate gains than the targeted group, across the other versions of $\Bias$, as well.
\input{graphs/income_states_targeted.tex}
On LAR, similar results hold. (Graphs and discussion are in the appendix.)

\subsection{Approximate Approach}\label{sec:res_approx}
\begin{tcolorbox}[boxsep=1pt,left=3pt,right=3pt,top=2pt,bottom=2pt]
\textbf{Key insights:}  Using the approximate approach greatly reduces precision in proving dataset-multiplicity robustness, but still shows promise for understanding dataset multiplicity's impact given low levels of label errors.  
\end{tcolorbox}
As expected, the approximate approach from \Cref{sec:approx} is less precise than the exact one. The loss in precision depends highly on the dataset and the level of label uncertainty, as shown in \Cref{fig:exactapproxcompare}.
For Income and MNIST, there are very large gaps: for example, given 2\% label error, 80\% of test samples are robust to dataset multiplicity, but the approximate version cannot prove robustness for any samples. However, there are some bright points: at 1\% label error, we can still prove robustness for 90\% of MNIST-1/7 samples, and over 60\% of Income samples. In situations where label error is expected to be relatively small, the approximate approach can still be useful for gauging the relative dataset multiplicity robustness of a dataset.

% We use this tradeoff because it makes the approximate approach, in particular, more feasible, and the 2-percentage-point drop in accuracy is within the generalization error for many algorithms. 

\input{graphs/exactapproxcompare}

We also measured the time complexity of each approach.  To check the robustness of \num{1000} Income samples, it takes 37.2 seconds for the exact approach and 6.8 seconds for the approximate approach. For \num{10000} samples, it takes 383.5 seconds and 30.7 seconds, respectively. I.e., the exact approach scales linearly with the number of samples, but the approximate approach stays within a single order of magnitude. See appendix for more details and discussion.

\section{Implications and Ethical Challenges of Dataset Multiplicity}\label{sec:implications}
For the conclusions we draw from machine learning to be robust and generalizable,
we need to 
\emph{understand} dataset multiplicity, \emph{reduce} its impacts on predictions, and \emph{adapt} machine learning practices to consider dataset multiplicity.

\paragraph{Understanding dataset multiplicity}
Adopting standardized data documentation practices will likely aid in identifying potential inaccuracies or biases in datasets~\cite{datasheets,datacards}. Further work surrounding \emph{how} and \emph{why} datasets are created with particular worldviews (e.g.,~\cite{hutchinsonAccountability,scheuermanPolitics}) will assist in identifying blind spots in existing datasets and help further the push for more robust dataset curation and documentation. 
But even if specific shortcomings in data collection are addressed through better curation and documentation practices, unavoidable variations in data collection will still contribute to dataset multiplicity~\cite{rechtImagenet}, making modeling and model deployment interventions important, too.

\paragraph{Reducing the impacts of dataset multiplicity}
An advantage of predictive multiplicity (i.e., multiplicity in the model selection process given a fixed training dataset) is that the wide range of equally-accurate models allows model developers to choose a model based on criteria like fairness or robustness without sacrificing predictive accuracy~\cite{blackMultiplicity, xin2022exploring}.  
Likewise, we know that there exist datasets --- typically de-biased or more representative than a na\"ively-collected baseline dataset --- that yield models that are both fair and accurate~\cite{chenWhy,duttaTradeoff,wickFairness}. 
If there exists a dataset in the dataset multiplicity set that yields a fairer (or more robust, more interpretable, etc.) model, then we should consider whether it is more appropriate to use that dataset to train the deployed model. 
(This may or may not be appropriate, depending on the domain, and is a decision that should be considered in conjunction with stakeholders.)
Another option is to find learning algorithms or model classes that are inherently more robust to dataset multiplicity. 
In the context of ridge regression, we found that using a larger regularization parameter (see \Cref{fig:tolerance}) increases dataset-multiplicity robustness. 
Ensemble learning is another promising direction, as it has been shown to decrease predictive multiplicity~\cite{black2022selective}. 

\paragraph{Adapting ML to handle non-robustness to dataset multiplicity}
If a model has low dataset multiplicity robustness in aggregate across a test dataset, then confidence may be too low to deploy the model because of procedural fairness concerns~\cite{blackMultiplicity}.   
It is also important to consider how robustness to dataset multiplicity varies across different population subgroups. 
As we saw in \Cref{sec:results_demo}, different multiplicity definitions can yield disparate multiplicity-robustness rates across populations.
The approximate approach from \Cref{sec:approx} is well-suited to these aggregate analyses.
If we find that dataset multiplicity rates are low (either overall or for some demographic groups), it may be more appropriate to train with a different algorithm, refine the training dataset so that multiplicity is lessened, or avoid machine learning for the task at hand.

Dataset multiplicity robustness should also be considered on an individual level, as in the exact approach from \Cref{sec:exact}.
If a given dataset and algorithm are not dataset multiplicity-robust on a test sample $\xtest$, options include abstaining on $\xtest$ or using the most favorable outcome in its dataset multiplicity range. 
But in many cases, an algorithm is deployed to allocate a finite resource --- thus, returning the best-case label for all samples is likely infeasible. Instead, the chosen model is a function of the arbitrary nature of the provided dataset. 
Creel and Hellman explain that arbitrary models are not necessarily cause for concern, however,  algorithmic monoculture becomes a concern when the same arbitrary outcomes are used widely, thereby broadly excluding otherwise-qualified people from opportunities~\cite{creel_hellman_2022, algorithmicMonoculture}. 
Given machine learning's reliance on benchmark datasets, it seems plausible that there is algorithmic monoculture stemming from the choice of arbitrary training dataset. 
To avoid algorithmic monoculture effects in the modeling process, scholars have proposed randomizing over model selection or outcomes~\cite{grgichlacaFairness,algorithmicMonoculture} --- we suspect that a similar approach would make sense in the context of dataset multiplicity. 

\section{Related Work}\label{sec:related}

Dataset multiplicity robustness can be used either to \emph{certify} (i.e., prove) that individual predictions are stable given uncertainty in the training data, or to characterize the overall stability of a model. 
Other approaches in the realms of  adversarial ML, uncertainty quantification, and learning theory aim to answer similar questions. 

\paragraph{Comparison with other sources of multiplicity}
Predictive multiplicity and underspecification show that there are often many models that fit a given dataset equally well~\cite{breiman-stability, amour-underspecification, predictive-multiplicity,teneyUnderspecification}. 
Because of this multiplicity, models can often be selected to simultaneously achieve accuracy and also fairness, robustness, or other desirable model-level properties~\cite{blackMultiplicity,costonFairness,semenovaExistence,teneyUnderspecification,xin2022exploring}. 
The extent of predictive multiplicity can be lessened by constructing more sophisticated models (e.g.,~\cite{roth2022reconciling}), however, this type of intervention only reduces algorithmic multiplicity and, furthermore, does not address the underlying procedural fairness concern that individuals can justifiably receive different decisions.
Multiplicity also arises when modifying training parameters like random seed, data ordering, and hyperparameters~\cite{bouthillier,cooperHyperparameter,mehrer2020individual,shumailovSGD,snoek2012practical}, but
most of the works on this topic focus on either attack vulnerability or the reproducibility and generalizability of the training process. 
In a notable exception, Bell et al. explore the `multiverse' of models by characterizing what hyperparameter values correspond with what conclusions~\cite{bell2022multiplicity}, but their analysis does not account for uncertainty in the training dataset, nor does the predictive multiplicity literature. 
There is, however, a line of work that aims to increase the fairness of a model by debiasing or augmenting a dataset~\cite{chenWhy,duttaTradeoff,wickFairness}. 
Our dataset multiplicity framework is more broad than those approaches since we aim to understand the entire range of feasible datasets and models, rather than identify a single fair alternative.

\paragraph{Other approaches to bounding uncertainty} Approaches from causal inference, uncertainty quantification, and learning theory aim to measure and reduce uncertainty in machine learning. 
One major concern in this realm is the role that researcher decisions can play in reproducibility~\cite{coker2021theory,simmonsPsychology}. 
Coker et al. propose `hacking intervals' to capture the range of outcomes that any realistic researcher could obtain through different analysis choices or datasets~\cite{coker2021theory}. 
Our dataset multiplicity framework can be viewed as extending their prescriptively constrained hacking-interval concept to allow for arbitrarily-defined changes to the training dataset. 
However, their results employ causality to make a stronger case for defining reasonable dataset modifications. 
Likewise, partial identification in economics uses domain knowledge and statistical tools to bound the range of possible outcomes in a data analysis~\cite{tamer2010partial}.

The methods described above --- `hacking intervals' and partial identification --- are special formulations of uncertainty quantification (UQ), which aims to understand the range of predictions that a model may output. 
UQ can occur either through Bayesian methods that treat model weights as random variables, or through ensembling or bootstrapping~\cite{sullivan2015introduction}. 
While UQ shares a common goal with dataset multiplicity --- i.e., understanding the range of outcomes --- the assumptions about where the multiplicity arises are different. 
UQ typically assumes that uncertainty stems from either insufficient data or noisy data, and does not account for the systemic errors that dataset multiplicity can encompass. 
Work on selection bias aims to learn in the presence of missing data, feature, or labels.  
For example, multiple imputation fills the missing data in multiple ways and aggregates the results~\cite{rubinImputation,van2018flexible}, similar to how dataset multiplicity considers all alternative models. 
The main difference between multiple imputation as a selection bias intervention and dataset multiplicity is that given selection bias, it is easy to identify 
where the inaccuracies are, and multiple imputation only considers a small number of dataset options, rather than all options as dataset multiplicity aims to do.

Within learning theory, distributional robustness studies how to find models that perform well across a family of distributions~\cite{ben-tal-distribution, namkoong-distribution, shafieezadeh-distribution}, robust statistics shows how algorithms can be adapted to account for outliers or other errors in the data~\cite{diakonikolas2019recent, diakonikolas_robustness}, and various works focus on robustifying training algorithms to label noise~\cite{natarajan, patrini, rolnick}. 
However, these works all (a) provide \textit{statistical} global robustness guarantees, rather than the provable \textit{exact} robustness guarantees that we make, (b) try to find a single good classifier, rather than understand the range of possible outcomes, and (c) typically require strong assumptions about the data distribution and the types of noise or errors.  
Algorithmic stability \cite{breiman-stability, devroye-stability} and sensitivity analysis \cite{hadi2009sensitivity} aim to quantify how sensitive algorithms are to small perturbations in the training data. 
However, they both typically make strict assumptions about the perturbation's form, either as a leave-one-out perturbation model in algorithmic stability~\cite{black-leaveoneout, kearns-stability}, or as random noise in sensitivity analysis. 
% Dataset multiplicity is clearly more broad then this, as dataset perturbations are not size-limited and generally can take more forms. 
% Our work, solving dataset multiplicity for linear models, differs from the algorithmic stability literature in a couple ways: first, rather than the commonly used leave-one-out perturbation model~\cite{black-leaveoneout, kearns-stability}, we allow for a fixed number of small label perturbations; second, as discussed above, our focus is proving point-wise robustness rather than ensuring comparable overall behavior.

\paragraph{Robustness in adversarial ML}
Checking robustness to dataset multiplicity has parallels to adversarial machine learning, especially data poisoning, where an attacker modifies a small portion of the training dataset to reduce test-time accuracy~\cite{biggio2012poisoning, shafahi2018poison, zhang-label-flipping, xiao-label-flipping}.
Various defenses counteract these attacks~\cite{jia-bagging, paudice-label-sanitation,rosenfeld-randomized, steinhardt-poisoning,zhang-debugging}, including ones that focus on attacking and defending linear regression models~\cite{jagielski-poisoning, muller-poisoning}. 
Some of these works (e.g., \cite{rosenfeld-randomized} for label-flipping) additionally try to \emph{certify} robustness. 
Our exact solution to dataset multiplicity in linear models with label errors functions as a certificate, since we prove robustness to all allowable label perturbations, including adversarial ones. 
Three major differences from Rosenfeld et al. \cite{rosenfeld-randomized} are that we do not modify the underlying algorithm to achieve a certificate, the certification process is deterministic, not probabilistic, and we allow the label perturbations to be targeted towards a particular subgroup. 
Meyer et al. use a similar targeted view on data modifications,
but their approach is strictly overapproximate and is limited to decision trees \cite{meyer}. 
Furthermore, our definition of dataset multiplicity is distinct from the \emph{defense} papers in that we aim to study dataset multiplicity robustness of existing algorithms; 
however, an interesting direction for future work would be to improve dataset multiplicity robustness via algorithmic modifications.

\section{Conclusions}
We defined the dataset multiplicity problem, showed how to evaluate the impacts of dataset multiplicity on linear models in the presence of label noise, and presented thoughts for how dataset multiplicity should be considered as part of the machine learning pipeline. Notably, we find that we can certify robustness to dataset multiplicity for some test samples, indicating that we can deploy these predictions with confidence. By contrast, we show that other test samples are not robust to low levels of dataset multiplicity, meaning that unless labels are very accurate, these test samples may receive predictions that are artifacts of the random nature of data collection, rather than real-world patterns.

Future work in the area of dataset multiplicity abounds, and many connections with other areas are mentioned throughout \Cref{sec:implications,sec:related}. The most important technical direction for future exploration, in our opinion, is extending the dataset multiplicity framework to probabilistic settings, e.g., by asking what proportion of reasonable datasets yield a different prediction for a given test sample. This inquiry is likely to be more fruitful than finding exact solutions for more complicated model classes, and it may open opportunities to leverage techniques from areas like distributional robustness or uncertainty quantification. 
Making direct connections with areas like causal inference and partial identification in economics should also be a priority for future work, as these topics have similar goals and have been studied more broadly.
In a social-science realm, dataset multiplicity could benefit from more work on  what features and labels in a dataset are most likely to be inaccurate or affected by social biases, since this will allow us to bound dataset multiplicity more precisely. Similarly, it is difficult to separate out instances of direct bias (e.g., salary disparities due to gender discrimination) and indirect bias (e.g., salary disparities due to women feeling unwelcome in STEM careers), and more research is needed into how that distinction should affect dataset multiplicity definitions.

\bibliographystyle{abbrv}
\bibliography{refs}

\newpage
\appendix

\section{Additional Details From Section~\ref{sec:linear}}\label{app:algs}

\paragraph{A note on the relationship between $\rr$ and $\Delta$} 
Given a fixed $\rr$, if $\delta_i$ is the same for all $i$, the ratio between $\delta$ and $\rr$ uniquely determines robustness to dataset
multiplicity (we make use of this fact for computing \Cref{tab:reg_data}).

\subsection{Details about Algorithm~\ref{alg:exact}}
\Cref{alg:exact-full} is the complete algorithm for the approach described in \Cref{sec:exact}. Note that this algorithm supersedes \Cref{alg:exact}. 

We define the \emph{positive potential impact} $\rho^+_i$ as the maximal positive change that perturbing $\yp_i$ can have on $\matrprod\y$, likewise, $\rho^-_i$ is the \emph{negative potential impact}, that is, the maximal negative change that perturbing $\yp_i$ can have on $\matrprod\y$.

\Cref{alg:exact-full} finds the minimal label perturbation necessary to change the label of test point $\mathbf{x}$, a fact we formalize in the following theorem:

\begin{theorem}\label{thm:alg_exact}
Suppose we have a deterministic learning algorithm $A$, a training dataset $D=(\x,\y)$ where up to $k$ labels $y_i$ corresponding to data points $\x_i$ that satisfy $\phi(\x_i)$ are inaccurate by $+/- \Delta$. Let $\Bias(D)$ be the set of all datasets that can be constructed by modifying $D$ according to $k$, $\phi$, and $\Delta$. Let $F=A(\Bias(D))$ be the set of models $f$ obtainable by training using $A$ on any dataset $D'\in\Bias(D)$. Given a test point $\xtest$, (i) \Cref{alg:exact-full} outputs an interval containing all values $f(\xtest)$ for all $f\in F$ (i.e., the algorithm is sound) and (ii) there is some $f_1, f_2 \in F$ such that $f_1(\xtest)$ is equal to the upper bound of the output and $f_2(\xtest)$ is equal to the lower bound of the output (i.e., the algorithm is tight). 
\end{theorem}

\begin{proof} We will prove (i) that the upper bound \Cref{thm:alg_exact}'s output is an upper bound on the value of $f(\xtest)$ for any $f\in F= A(\Bias(D))$ and (ii) that this upper bound is achieved by some $f_1\in F$. The proofs for the lower bounds are analogous. 

    (i) Let $u$ be the upper bound of \Cref{thm:alg_exact}'s output. We want to show that $f(\xtest)\leq u$ for all $f\in F=A(\Bias(D))$, where $$\Bias(D=(\x,\y))= \{(\x, \y')\mid \|\y'-\y\|_1 \leq k \land \y_i\neq \y'_i \implies \phi(\x_i) \land \|\y_i-\y'_i\|\leq \Delta\}$$.

    Suppose, towards contraction, that there is some $f'\in A(\Bias(D))$ such that $f'(\xtest)=u' > u$. Then, there is some set of labels $y_{i_1}, y_{i_2}, \ldots, y_{i_k}$ that can be modified to create a $\y'$ such that $(\x^T\x)^{-1}\x^T\y'\xtest = u'$. Recall that $z=\xtest(\x^T\x)^{-1}\x^T$. So, either (a), we modify the same set of labels, but modify at least one in a different magnitude or direction, or (b), there must be some ${i_j}$ that we modify $y_{i_j}$ to find $u'$, but the algorithm does not identify this index in line 8 of the algorithm. 

    First, suppose (a) occurred. WLOG, suppose the if case on line 2 is satisfied, i.e., $z_{i_j}\geq 0$. Then we hypothetically modified $y_{i_j}$ by $a\neq \delta^u_{i_j}$ in place of line 11. We know $a < \delta^u_{i_j}$ since $\delta$ is an upper bound on how much we can change each label. We have $\delta^u_{i_j}\geq 0$ and $z_{i_j}\geq 0$, so their product is also greater than 0, so $az_{i_j} < \delta^u_{i_j}z_{i_j}$. So, the final product $\mathbf{z}\y'$ cannot be larger than had we followed the algorithm.

    Now, suppose (b) occurred. Suppose $y_{i_j}$ is modified to yield $u'$, but is not modified in the algorithm. Then, there must be some $y_{i'}$ such that (assume WLOG that $\mathbf{z}_{i'}\geq 0$ and $\mathbf{z}_{i_j}\geq 0$) $\mathbf{z}_{i'}\delta^u_{i'} \geq \mathbf{z}_{i_j}\delta^u_{i_j}$. If equality holds, we will have $u=u'$. So, assume $\mathbf{z}_{i'}\delta^u_{i'} > \mathbf{z}_{i_j}\delta^u_{i_j}$. But then, modifying $\y_{i'}$ by $\delta^u_{i'}$ will result in a greater increase to $mathbf{z}\y$ than increasing $\y_{i_j}$ by $\delta^u_{i_j}$ will. So, changing $y_{i_j}$ cannot result in an output $u'>u$. 
    
    (ii) We need to construct the function $f\in F=A(\Bias(D))$ such that $f(\xtest)=u$, where $u$ is the upper bound of \Cref{thm:alg_exact}'s output. Let $y_{i_1}, y_{i_2},\ldots, y_{i_k}$ be the labels modified by lines 15-19 of the algorithm to yield some $\y'$. Then, let $f(x) = (\x^T\x)^{-1}\x^T\y'$.  
\end{proof}

\subsection{Details on the Approximate Approach}
We will next present an example to show that $\thetaset$ can be non-convex.

\begin{example}\label{ex:approx1}
Suppose $\y=(1,-1,2)$, $\Delta=[-1,1]^3$, and $\l=2$. Given $\matrprodc=\begin{pmatrix}1 & 2 & 1 \\ -1 & 0 &2 \\ 2 &1 & 0 \end{pmatrix}$, we have 
$$ \thetaset =  \matrprodc\y \cup 
\left\{\matrprodc\begin{pmatrix}a \\ b \\2 \end{pmatrix} \right\} \cup
\left\{\matrprodc\begin{pmatrix}a \\ -1 \\c \end{pmatrix}\right\} \cup
\left\{ \matrprodc\begin{pmatrix}1 \\ b \\c  \end{pmatrix}\right\} $$

for $a\in[0,2]$, $b\in[-2,0]$, and $c\in[1,3]$. 

Note that $(3,6,3)^\top\in\thetaset$ and $(4,5,2)^\top\in\thetaset$, but their midpoint $(3.5, 5.5, 3.5)^\top\notin\thetaset$, thus, $\thetaset$ is non-convex.
\end{example}

We present the complete algorithm for procedure described in \Cref{sec:approx} \Cref{alg:approx}. Next, we will show that this algorithm outputs the tightest hyperrectangular (i.e., box) enclosure of $\Bias$.

\begin{theorem}
    \Cref{alg:approx} computes the tightest hyperrectangular enclosure of $\Bias$. 
\end{theorem}

\begin{proof}
First, we will show that \Cref{alg:approx} computes an enclosure of $\Bias$, and next we will show that this output is the tightest hyperrectangular enclosure of $\Bias$.

By construction of the algorithm, we see that the output will be an enclosure of $\Bias$. The algorithm constructs the maximal way to increase/decrease each coordinate.

Now, suppose there is another hyperrectangle ${{\thetaset}^a}'$ that with $\Bias(D)\subseteq{{\thetaset}^a}'$ and ${{\thetaset}^a}'_i\subset {\thetaset}^a_i$. 
WLOG, assume that the upper bound of ${\thetaset^a}'_i$ is strictly less than the upper bound of ${\thetaset}^a_i$. 
But ${\thetaset}^a_i=\max_{(\x,\y')\in\Bias(D)}(\matrprod_i\y')$, which means that $\y^*=\max_{(\x,\y')\in\Bias(D)}\matrprodc \y'$ has $\matrprodc_i\y^*$ greater than the upper bound of ${\thetaset^a}'_i$, and thus ${\thetaset^a}'$ is not a sound enclosure of $\thetaset$.
\end{proof}

\begin{algorithm}[ht]
    \caption{Solve for $V= [\min_{(\x,\yp)\in\Bias_{\l,\Delta,\phi}((\x,\y))} \matrprod\yp, \max_{(\x,\yp)\in\Bias_{\l,\Delta,\phi}((\x,\y))} \matrprod\yp]$ by finding perturbations of $\y$ that maximally decrease/increase $\matrprod\y$.}
    \begin{algorithmic}[1]
        \Require $\matrprod\in\mathbb{R}^n, (\x,\y)\in(\mathbb{R}^{n\times d},\mathbb{R}^n), \Delta\in\intdomain^n$ with $0\in\Delta$, $\l\geq 0$, $\phi:\mathbb{R}^d\rightarrow\{0,1\}$        \State $\yl\gets\y$ and $\yu\gets\y$
        \If{$z_i\geq 0$}
            \State $\rho^+_i \gets \matrprod_i\delta_i^u\ $, $\ \ \rho^-_i \gets \matrprod_i\delta_i^l$
        \Else 
            \State $\rho^+_i \gets \matrprod_i\delta_i^l\ $,  $\ \ \rho^-_i\gets \matrprod_i\delta_i^l$
        \EndIf
        \If{not $\phi(\mathbf{x}_i)$}
            \State $\rho^+_i\gets 0$, $\ \ \rho^-_i \gets 0$
        \EndIf
        \State Let $\rho^+_{i_1},\ldots,\rho^+_{i_l}$ be the $\l$ largest elements of $\rho^+$ by absolute value
        \ForEach{$\rho^+_{i_j}$}
            \If{$z_{i_j}\geq 0$}
                \State $(\yuscal)_{i_j} \gets (\yuscal)_{i_j} + \delta_{i_j}^u$
            \Else
                \State $(\yuscal)_{i_j}\gets (\yuscal)_{i_j} + \delta_{i_j}^l$
            \EndIf
        \EndFor
        \State Let $\rho^-_{i_1},\ldots,\rho^-_{i_l}$ be the $\l$ largest elements of $\rho^-$ by absolute value
        \ForEach{$\rho^-_{i_j}$}
            \If{$z_{i_j}\geq 0$}
                \State $(\ylscal)_{i_j}\gets (\ylscal)_{i_j} + \delta_{i_j}^l $
            \Else
                \State $(\ylscal)_{i_j} \gets (\ylscal)_{i_j} + \delta_{i_j}^u$
            \EndIf
        \EndFor
        \State $V=[\matrprod\yl, \matrprod\yu]$
    \end{algorithmic}\label{alg:exact-full}
\end{algorithm}

\begin{algorithm}
    \caption{Computing $\thetaset^a$}
    \begin{algorithmic}[1]
        \Require $\matrprodc\in\mathbb{R}^{d\times n}, (\x,\y)\in(\mathbb{R}^{n\times d},\mathbb{R}^n), \Delta\in\intdomain^n$ with $0\in\Delta$, $\l\geq 0$, $\phi:\mathbb{R}^d\rightarrow\{0,1\}$        \State $\yl\gets\y$ and $\yu\gets\y$

        \State $\ \thetaset^a\gets[0,0]^d$
        \For{$i$ \textrm{ in range} $d$}:
            \State $\yl\gets\y$, $\ \ \yu\gets\y$
            \If{$c_{ij} < 0$}
                \State  $\rho^+_j \gets \matrprodc_{ij}\delta_j^l\ $,  $\ \ \rho^-_j\gets \matrprodc_{ij}\delta_j^u$
            \Else
                \State $\rho^+_j\gets\matrprodc_{ij}\delta_j^u\ $,  $\ \ \rho^-_j\gets \matrprodc_{ij}\delta_j^l$
            \EndIf
            \If{not $\phi(\mathbf{x}_i)$}
                \State $\rho^+_j\gets 0$, $\ \ \rho^-_j\gets 0$
            \EndIf
            \State Let $\rho^+_{k_1},\ldots,\rho^+_{k_l}$ be the $\l$ largest elements of $\rho^+$ by absolute value.
            \ForEach{$\rho^+_{k_j}$}
                \If{$c_{i{k_j}}\geq 0$}
                    \State $(\yuscal)_{k_j} \gets (\yuscal)_{k_j} + \delta_{k_j}^u$
                \Else
                    \State $(\yuscal)_{k_j} \gets (\yuscal)_{k_j} + \delta_{k_j}^l$
                \EndIf
            \EndFor
            \State Let $\rho^-_{k_1},\ldots,\rho^-_{k_l}$ be the $\l$ largest elements of $\rho^-$ by absolute value.
            \ForEach{$\rho^-_{k_j}$}
                \If{$c_{i{k_j}} \geq 0$}
                    \State $(\ylscal)_{k_j} \gets (\ylscal)_{k_j} + \delta_{k_j}^l$
                \Else
                    \State $(\ylscal)_{k_j} \gets (\ylscal)_{k_j} + \delta_{k_j}^u$
                \EndIf
            \EndFor
        
            \State $\thetaset^a_i \gets [\mathbf{c}_i\yl, \mathbf{c}_i \yu]$
        \EndFor
    \end{algorithmic}\label{alg:approx}

\end{algorithm}

\section{Additional Experiments}
\label{app:more-experiments}

We present additional tables, graphs, and discussion about the experimental results. 

\paragraph{Accuracy}
The maximal accuracy for each dataset (i.e., in \cref{fig:tolerance}, the accuracy for the 0.0 line) is 76.5\% for Income, 61.9\% for LAR, and 98.9\% for MNIST. The exact values we used for $\lambda$ (as well as the procedure to obtain $\lambda$) are in the code.

\paragraph{Additional LAR data}
\Cref{fig:demo_rob_target_hmda} shows demographic-stratified dataset-multiplicity
robustness rates for LAR under different ways of defining targeted dataset multiplicity. To varying extents, the majority/advantaged group sees higher robustness rates as compared with the disadvantaged group across all versions of $\Bias$. 

\input{graphs/targeted_demo_robustness.tex}

\paragraph{Regression dataset results} \Cref{tab:reg_data} presents results on Income-Reg for the fixed robustness radius $\rr=\$\num{2000}$, which we chose as a challenging, but reasonable,  definition for two incomes being close. We also empirically validated that the ratio between $\Delta$ and $\epsilon$ uniquely determines robustness for a fixed multiplicity definition. Notably, for small $\Delta$ to $\rr$ ratios (i.e., when we can modify labels by small amounts, but predictions can be far apart and still considered robust), many test points are dataset-multiplicity robust, even when the number of untrustworthy training labels is relatively large (up to $10\%$). However, when this ratio is large, e.g., when $\Delta = 5\rr$, we are still able certify a majority of test points as robust up to 1\% bias. For example, if we can modify 1\% of labels by up to \$\num{10000}, then  69.6\% of test samples' predictions cannot be modified by more than \$\num{2000}. By contrast, only allowing 1\% of labels to be modified by up to \$\num{4000} yields a dataset-multiplicity robustness rate of 91.2\% (again, within a radius of \$\num{2000}). 

\begin{table}[t]
\small
\centering
\caption{Robustness rates (percentage of test dataset whose prediction cannot change by more than $\rr$) for Income-Reg given various $\Delta$ and $\l$ values. $\epsilon=\num{2000}$ in all experiments. Note that the shorthand $\Delta=a$ means $\Delta=[-a,a]^n$. Column 2 gives the ratio between the maximum label perturbation ($\Delta$) and the robustness radius ($\rr$), which uniquely determines robustness for a given $\l$. 
}
\label{tab:reg_data}
\begin{tabular}{rc|rrrrrrrrrr}\toprule
\multirow{2}{*}{$\Delta$}  & \multirow{2}{*}{Ratio $\frac{\Delta}{\epsilon}$} & \multicolumn{10}{c}{Maximum label error $\l$ as a percentage of training dataset size} \\ 
    &  & 1.0\%   & 2.0\%   & 3.0\%  & 4.0\%   & 5.0\%  & 6.0\% & 7.0\%  & 8.0\%  & 9.0\%  & 10.0\% \\ \midrule
\num{1000} & 0.5  & 100.0 & 100.0 & 100.0 & 100.0 & 99.8 & 99.4 & 99.1 & 98.6 & 98.0 & 97.1 \\
\num{2000} & 1    & 100.0 & 96.6  & 91.1  & 85.4  & 84.1 & 76.4 & 73.3 & 64.0 & 49.9 & 35.2 \\
\num{4000} & 2    & 91.2  & 84.2  & 69.2  & 35.7  & 14.2 & 2.0  & 0    & 0    & 0    & 0    \\
%\num{5000} & 2.5  & 87.9  & 76.2  & 37.8  & 10.2  & 0    & 0    & 0    & 0    & 0    & 0    \\
\num{6000}  & 3    & 85.9  & 61.1  & 15.4  & 0     & 0    & 0    & 0    & 0    & 0    & 0    \\
\num{8000}  & 4    & 80.1  & 20.0  & 0     & 0     & 0    & 0    & 0    & 0    & 0    & 0    \\
\num{10000}  & 5   & 69.6  & 3.6   & 0     & 0     & 0    & 0    & 0    & 0    & 0    & 0    \\ \bottomrule   
\end{tabular}
\end{table}

\paragraph{Income demographics} \Cref{tab:statesdemo} shows the demographic make-up of various states' data from the Income dataset. Notice that Oregon has the lowest percentage of Black people in the dataset - we suspect that this is why the robustness rates between White and Black people is so large.

\begin{table}[t]
\small
\centering
\caption{Summary of data download by state from the Folktables Income task.}
\label{tab:statesdemo}
\begin{tabular}{l|rrr}\toprule
State & training $n$ & \% White & \% Black \\ \midrule
Georgia & 40731 & 67.6 & 23.9 \\
Louisiana & 16533 & 70.9 & 23.5 \\
Maryland & 26433 & 63.6 & 23.5 \\
Oregon & 17537 & 86.4 & 1.4 \\
Wisconsin & 26153 & 92.7 & 2.6 \\
\bottomrule   
\end{tabular}
\end{table}

\Cref{fig:statesdemo} shows robustness rates, stratified by race or gender, for five U.S. states on the Income dataset. We see that for most states, there is a significant gap in robustness rates across with race and gender. In particular, Georgia and Louisiana have higher robustness rates for Black people, and Oregon has \emph{drastically} higher robustness rates for White people. All states, except Wisconsin, have higher robustness rates for men than women. 

\input{graphs/income_states.tex}

\Cref{fig:approx_demo_income} shows the demographic group-level robustness rates under the over-approximate technique. 

\input{graphs/approx_demo.tex}

\subsection{Running time}
\Cref{tab:timedata} shows the running time of our techniques, as evaluated on a 2020 MacBook Pro with 16GB memory and 8 cores. These times should be interpreted as upper bounds; in practice, both approaches are amenable to parallelization, which would yield faster performance.  We notice that both the exact approach scales linearly with the number of samples, while the approximate approach stays within a single order of magnitude as the number of samples grows.Clearly, the approximate approach is more scalable for checking the robustness of large numbers of data points. 

\begin{table}[t]
\small
\centering
\caption{Running time, in seconds, for exact and approximate experiments. The exact experiments flip labels for each data point until the sample is no longer robust. The approximate experiments flip 10\% of the labels (which is enough to bring the robustness to 0\%).}
\begin{tabular}{l|rrrrrr}\toprule
\multirow{2}{*}{Dataset} &  \multicolumn{2}{c}{\num{100} samples} & \multicolumn{2}{c}{\num{1000} samples} & \multicolumn{2}{c}{\num{10000} samples} \\
  &   Exact & Approx. & Exact & Approx. & Exact & Approx. \\\midrule
Income   & 2.5 & 4.4 & 37.2 & 6.8  & 383.5 & 30.7 \\
LAR    & 7.5 & 1.8&   73.5 &   2.6 &  730.2 & 10.6 \\
MNIST 1/7  & 4.1 & 3.8 & 24.8 & 6.0 &  448.5 & 24.3 \\\bottomrule                          
\end{tabular}\label{tab:timedata}
\end{table}

\Cref{tab:milp_data} shows the running time of the MILP solver for Income and LAR. We see that the times scale linearly (as with our exact approach), but are much worse. To check robustness for 100 samples, it is over 80\% slower to use MILP.

\begin{table}[t]
\small
\centering
\caption{Running time, in seconds, for the MILP solver. We modify 10\% of the labels. }
\begin{tabular}{l|rr}\toprule
Dataset & 10 samples & 100 samples \\\midrule
Income   & 49.1 & 498.8 \\
LAR    & 58.1 & 616.6  \\
\end{tabular}\label{tab:milp_data}
\end{table}

\end{document}